\documentclass{article}
\usepackage{arxiv}
\usepackage{microtype}
\usepackage[utf8]{inputenc} 
\usepackage[T1]{fontenc}    
\usepackage{nicefrac}
\usepackage{lipsum}
\usepackage{graphicx}
\usepackage{subfigure}
\usepackage{amsfonts}
\usepackage{commath}
\usepackage{amsmath}
\allowdisplaybreaks
\usepackage{dirtytalk}
\usepackage{amsthm}
\usepackage{mathtools}
\usepackage[shortlabels]{enumitem}

\usepackage{natbib}

\DeclareMathOperator*{\argmin}{arg\,min}

\newtheorem{theorem}{Theorem}[section]
\newtheorem{lemma}[theorem]{Lemma}
\newtheorem{proposition}[theorem]{Proposition}

\newtheorem{assumption}{Assumption}
\newtheorem{remark}[theorem]{Remark}
\theoremstyle{definition}
\newtheorem{definition}[theorem]{Definition}

\usepackage{tikz-cd}
\usepackage{booktabs} 
\usepackage{hyperref}

\title{Regularised Least-Squares Regression with Infinite-Dimensional Output Space}
\author{Junhyung Park\thanks{Corresponding author: \texttt{junhyung.park@tuebingen.mpg.de}}\\
	Max Planck Institute for Intelligent Systems\\
	T\"ubingen, Germany\\
	\And
	Krikamol Muandet\\
	Max Planck Institute for Intelligent Systems\\
	T\"ubingen, Germany}
\date{}
\begin{document}
\maketitle
\begin{abstract}
This short technical report presents some learning theory results on vector-valued reproducing kernel Hilbert space (RKHS) regression, where the input space is allowed to be non-compact and the output space is a (possibly infinite-dimensional) Hilbert space. Our approach is based on the integral operator technique using spectral theory for non-compact operators. We place a particular emphasis on obtaining results with as few assumptions as possible; as such we only use Chebyshev's inequality, and no effort is made to obtain the best rates or constants. 
\end{abstract}
\section{Introduction}\label{Sintro}
Much (but not all) of the work in the learning theory of kernel ridge regression with regularised least-squares risk has been on a compact input space and a real output space. In this work, we extend, where possible, the results to the case where the input space is not restricted to be compact and the output space is a (possibly infinite dimensional) Hilbert space, for which there has been an increase in interest recently. In this case, the RKHS in which the regression is carried out is a vector-valued RKHS with an operator-valued kernel \cite{micchelli2005learning,carmeli2006vector,carmeli2010vector}. 

We focus on obtaining universal consistency of vector-valued RKHS regression with minimal assumptions on the learning problem. While we also provide uniform rates for the well-specified case, no further effort is made to find assumptions under which we can obtain best rates or best constants. Our approach is based on the integral operator technique, but in order to account for the fact that the output space is an infinite-dimensional vector space, we leverage spectral theory of non-compact operators \citep[Chapter 7]{hall2013quantum}. 

\section{Notations, Background and Problem Set-Up}\label{Snotations}
Let us take \((\Omega,\mathcal{F},P)\) as the underlying probability space. Suppose \((\mathcal{X},\mathfrak{X})\) is a separable measurable space, and that \(\mathcal{Y}\) is a (potentially infinite-dimensional) separable Hilbert space with associated inner product and norm denoted by \(\langle\cdot,\cdot\rangle_\mathcal{Y}\) and \(\lVert\cdot\rVert_\mathcal{Y}\). Denote the Borel \(\sigma\)-algebra of \(\mathcal{Y}\) as \(\mathfrak{Y}\). Suppose \(X:\Omega\rightarrow\mathcal{X}\) and \(Y:\Omega\rightarrow\mathcal{Y}\) are random variables, with distributions \(P_X(A)=P(X^{-1}(A))\) for \(A\in\mathfrak{X}\) and \(P_Y(B)=P(Y^{-1}(B))\) for \(B\in\mathfrak{Y}\). Further, we denote by \(P_{XY}\) the joint distribution of \(X\) and \(Y\). In order for regression of \(Y\) on \(X\) to be possible, the following assumption that \(Y\) has finite variance is a minimal requirement:
\begin{assumption}\label{Afinitevariance}
	We have \(\mathbb{E}\left[\left\lVert Y\right\rVert^2_\mathcal{Y}\right]<\infty\). 
\end{assumption}
Assumption \ref{Afinitevariance} also implies that \(\mathbb{E}[\lVert Y\rVert_\mathcal{Y}]<\infty\), which means that \(Y\) is Bochner-integrable \citep[p.15, Definition 35]{dinculeanu2000vector}. Hence, we can define its conditional expectation \(\mathbb{E}[Y\mid X]\) as an \(X\)-measurable, Bochner-\(P_X\)-integrable random variable taking values in \(\mathcal{Y}\), according to \citet[p.45, Definition 38]{dinculeanu2000vector}. In the rest of this paper, we let \(\mathbb{E}[Y\mid X]\) be any particular version thereof, and talk about \textit{the} conditional expectation of \(Y\) given \(X\). Since \(\mathbb{E}\left[Y\mid X\right]\) is an \(X\)-measurable random variable, we can write
\begin{equation}\label{EdefinitionF}
	\mathbb{E}[Y\mid X]=f^*(X).
\end{equation}
for some deterministic measurable function \(f^*:\mathcal{X}\rightarrow\mathcal{Y}\). It is this function \(f^*\) that we aim to estimate via regression. 

Denote by \(L^2(\mathcal{X},P_X;\mathcal{Y})\) the Bochner space with output in \(\mathcal{Y}\), i.e. the Hilbert space of (equivalence classes of) measurable functions \(f:\mathcal{X}\rightarrow\mathcal{Y}\) such that \(\lVert f(\cdot)\rVert^2_\mathcal{Y}\) is \(P_X\)-integrable, with inner product \(\langle f_1,f_2\rangle_2=\mathbb{E}[\langle f_1(X),f_2(X)\rangle_\mathcal{Y}]\). Denote its corresponding norm by \(\lVert\cdot\rVert_2\). Then by Jensen's inequality and Assumption \ref{Afinitevariance}, we have \(f^*\in L^2(\mathcal{X},P_X;\mathcal{Y})\):
\[\mathbb{E}\left[\left\lVert f^*(X)\right\rVert^2_\mathcal{Y}\right]=\mathbb{E}\left[\left\lVert\mathbb{E}\left[Y\mid X\right]\right\rVert^2_\mathcal{Y}\right]\leq\mathbb{E}\left[\mathbb{E}\left[\left\lVert Y\right\rVert^2_\mathcal{Y}\mid X\right]\right]=\mathbb{E}\left[\left\lVert Y\right\rVert^2_\mathcal{Y}\right]<\infty.\]

\subsection{Vector-Valued Reproducing Kernel Hilbert Spaces}\label{SSvvrkhs}
In this report, regression for \(f^*\in L^2(\mathcal{X},P_X;\mathcal{Y})\) will be carried out in a \textit{fixed} vector-valued reproducing kernel Hilbert space, the well-known theory of which we briefly review here. 

Suppose that \(\mathcal{H}\) is a Hilbert space of functions \(\mathcal{X}\rightarrow\mathcal{Y}\), with inner product and norm denoted by \(\langle\cdot,\cdot\rangle_\mathcal{H}\) and \(\lVert\cdot\rVert_\mathcal{H}\) respectively. For any \(n\in\mathbb{N}\), we denote by \(\mathcal{X}^n\) and \(\mathcal{Y}^n\) the \(n\)-fold direct sums of \(\mathcal{X}\) and \(\mathcal{Y}\) respectively; in particular, \(\mathcal{Y}^n\) is a Hilbert space with inner product \(\langle(y_1,...,y_n)^T,(y'_1,...,y'_n)^T\rangle_{\mathcal{Y}^n}=\sum^n_{i=1}\langle y_i,y'_i\rangle_\mathcal{Y}\). For any \(\mathbf{x}=(x_1,...,x_n)^T\in\mathcal{X}^n\), we define the \textit{evaluation operator} (or \textit{sampling operator}) by
\begin{alignat*}{2}
	S_\mathbf{x}:&\mathcal{H}\rightarrow\mathcal{Y}^n\\
	&f\mapsto\frac{1}{n}\left(f\left(x_1\right),...,f\left(x_n\right)\right)^T.
\end{alignat*}
Then \(\mathcal{H}\) is a \textit{vector-valued reproducing kernel Hilbert space} (vvRKHS) if the evaluation map \(S_x:\mathcal{H}\rightarrow\mathcal{Y}\) is continuous for all \(x\in\mathcal{X}\) \citep[Definition 2.1]{carmeli2006vector}. This immediately implies that \(S_\mathbf{x}:\mathcal{H}\rightarrow\mathcal{Y}^n\) is continuous for all \(n\in\mathbb{N}\) and \(\mathbf{x}\in\mathcal{X}^n\). We define the \textit{operator-valued kernel} \(K:\mathcal{X}\times\mathcal{X}\rightarrow\mathcal{L}(\mathcal{Y})\), where \(\mathcal{L}(\mathcal{Y})\) is the Banach space of continuous linear operators from \(\mathcal{Y}\) to itself, by
\[K\left(x,x'\right)(y)=S_xS_{x'}^*y,\qquad\text{i.e.}\qquad K\left(\cdot,x'\right)(y)=S_{x'}^*(y).\]
Then we can easily deduce the \textit{reproducing property}. For any \(x\in\mathcal{X}\) and \(y\in\mathcal{Y}\),
\[\left\langle y,f(x)\right\rangle_\mathcal{Y}=\left\langle y,S_x(f)\right\rangle_\mathcal{Y}=\left\langle S_x^*(y),f\right\rangle_\mathcal{H}=\left\langle K(\cdot,x)(y),f\right\rangle_\mathcal{H}.\]
For arbitrary \(n\in\mathbb{N}\) and \(\mathbf{x}=(x_1,...,x_n)^T\in\mathcal{X}^n\),
the adjoint of the sampling operator, \(S^*_\mathbf{x}:\mathcal{Y}^n\rightarrow\mathcal{H}\), is given by
\[S^*_\mathbf{x}\mathbf{y}=\frac{1}{n}\sum^n_{i=1}K(x_i,\cdot)y_i,\qquad\text{for }\mathbf{y}=(y_1,...,y_n)^T,y_i\in\mathcal{Y},\]
since, by the reproducing property, for any \(f\in\mathcal{H}\) and \(\mathbf{y}\in\mathcal{Y}^n\),
\[\left\langle S_\mathbf{x}f,\mathbf{y}\right\rangle_{\mathcal{Y}^n}=\frac{1}{n}\sum^n_{i=1}\left\langle f\left(x_i\right),y_i\right\rangle_\mathcal{Y}=\frac{1}{n}\sum^n_{i=1}\left\langle f,K\left(x_i,\cdot\right)y_i\right\rangle_\mathcal{H}=\left\langle f,\frac{1}{n}\sum^n_{i=1}K\left(x_i,\cdot\right)y_i\right\rangle_\mathcal{H}.\]
\begin{assumption}\label{Aseparablebounded}
	We henceforth assume that \(\mathcal{H}\) is separable, and that the kernel \(K\) is bounded:
	\[\sup_{x\in\mathcal{X}}\left\lVert K(x,x)\right\rVert_\textnormal{op}=\sup_{x\in\mathcal{X}}\sup_{y\in\mathcal{Y},\left\lVert y\right\rVert_\mathcal{Y}\leq1}\left\lVert K(x,x)(y)\right\rVert_\mathcal{Y}<B,\qquad\text{for some }B>0.\]
\end{assumption}
For a fixed \(f\in\mathcal{H}\), Assumption \ref{Aseparablebounded} allows us to bound \(\lVert f(\cdot)\rVert_\mathcal{Y}\) uniformly over \(\mathcal{X}\), and hence the operator norm of \(S_\mathbf{x}\) uniformly over \(\mathcal{X}^n\). 
\begin{lemma}\label{Lbound}
	Suppose Assumption \ref{Aseparablebounded} holds. Then
	\begin{enumerate}[(i)]
		\item For all \(f\in\mathcal{H}\),
		\[\sup_{x\in\mathcal{X}}\left\lVert f(x)\right\rVert_\mathcal{Y}\leq\sqrt{B}\left\lVert f\right\rVert_\mathcal{H}.\]
		\item For all \(n\in\mathbb{N}\), 
		\[\sup_{\mathbf{x}\in\mathcal{X}^n}\left\lVert S_\mathbf{x}\right\rVert_\textnormal{op}^2\leq\frac{B}{n}.\]
	\end{enumerate}
\end{lemma}
\begin{proof}
	\begin{enumerate}[(i)]
		\item We use the reproducing property and the Cauchy-Schwarz inequality repeatedly to obtain:
		\begin{alignat*}{2}
			\left\lVert f(x)\right\rVert^2_\mathcal{Y}=\left\langle f(x),f(x)\right\rangle_\mathcal{Y}=\left\langle f,K(\cdot,x)(f(x))\right\rangle_\mathcal{H}&\leq\left\lVert f\right\rVert_\mathcal{H}\left\langle K(\cdot,x)(f(x)),K(\cdot,x)(f(x))\right\rangle_\mathcal{H}^{1/2}\\
			&=\left\lVert f\right\rVert_\mathcal{H}\left\langle f(x),K(x,x)(f(x))\right\rangle_\mathcal{Y}^{1/2}\\
			&\leq\left\lVert f\right\rVert_\mathcal{H}\left\lVert f(x)\right\rVert_\mathcal{Y}^{1/2}\left\lVert K(x,x)\left(f(x)\right)\right\rVert_\mathcal{Y}^{1/2}\\
			&\leq\left\lVert f\right\rVert_\mathcal{H}\left\lVert f(x)\right\rVert_\mathcal{Y}\left\lVert K(x,x)\right\rVert_\textnormal{op}^{1/2}.
		\end{alignat*}
		Now divide both sides by \(\lVert f(x)\rVert_\mathcal{Y}\) and apply the bound in Assumption \ref{Aseparablebounded}. 
		\item We can apply (i) to obtain
		\[\sup_{\mathbf{x}\in\mathcal{X}^n}\left\lVert S_\mathbf{x}\right\rVert_\text{op}^2=\sup_{\mathbf{x}\in\mathcal{X}^n}\sup_{f\in\mathcal{H},\left\lVert f\right\rVert_\mathcal{H}\leq1}\left\lVert S_\mathbf{x}f\right\rVert_{\mathcal{Y}^n}^2=\sup_{\mathbf{x}\in\mathcal{X}^n}\sup_{f\in\mathcal{H},\left\lVert f\right\rVert_\mathcal{H}\leq1}\frac{1}{n^2}\sum^n_{i=1}\left\lVert f(x_i)\right\rVert^2_\mathcal{Y}\leq\frac{B}{n}.\]
	\end{enumerate}
\end{proof}
Lemma \ref{Lbound}(i) immediately implies that \(\mathcal{H}\subseteq L^2(\mathcal{X},P_X;\mathcal{Y})\), since, for any \(f\in\mathcal{H}\), \(\mathbb{E}\left[\left\lVert f(X)\right\rVert_\mathcal{Y}^2\right]\leq B\left\lVert f\right\rVert_\mathcal{H}^2<\infty\), and the inclusion \(\iota:\mathcal{H}\rightarrow L^2(\mathcal{X},P_X;\mathcal{Y})\) is a bounded linear operator with \(\left\lVert\iota\right\rVert_\text{op}\leq\sqrt{B}\):
\[\left\lVert\iota(f)\right\rVert_2=\sqrt{\mathbb{E}\left[\left\lVert f(X)\right\rVert^2_\mathcal{Y}\right]}\leq\sqrt{B}\left\lVert f\right\rVert_\mathcal{H},\qquad\text{for all }f\in\mathcal{H}.\]
Denote the adjoint of the inclusion by \(\iota^*:L^2(\mathcal{X},P_X;\mathcal{Y})\rightarrow\mathcal{H}\). Then \(\iota^*\circ\iota:\mathcal{H}\rightarrow\mathcal{H}\) and \(\iota\circ\iota^*:L^2(\mathcal{X},P_X;\mathcal{Y})\rightarrow L^2(\mathcal{X},P_X;\mathcal{Y})\) are self-adjoint operators.

Let \(\{(X_i,Y_i)\}^n_{i=1}\) be i.i.d. copies of \((X,Y)\), and denote by \(\mathbf{X}\) and \(\mathbf{Y}\) the random vectors \((X_1,...,X_n)^T\in\mathcal{X}^n\) and \((Y_1,...,Y_n)^T\in\mathcal{Y}^n\). Then the operators \(S_\mathbf{X}:\mathcal{H}\rightarrow\mathcal{Y}^n\) and \(S^*_\mathbf{X}:\mathcal{Y}^n\rightarrow\mathcal{H}\), given by \(S_\mathbf{X}(f)=\frac{1}{n}(f(X_1),...,f(X_n))^T\) and \(S^*_\mathbf{X}((y_1,...,y_n)^T)=\frac{1}{n}\sum^n_{i=1}K(X_i,\cdot)y_i\) respectively, are random. 
\begin{lemma}\label{Lintegralexpression}
	\begin{enumerate}[(i)]
		\item An explicit integral expression for \(\iota^*:L^2(\mathcal{X},P_X;\mathcal{Y})\rightarrow\mathcal{H}\) can be given as
		\[\iota^*\left(f\right)\left(\cdot\right)=\mathbb{E}\left[K\left(\cdot,X\right)f\left(X\right)\right]\qquad\text{for }f\in L^2(\mathcal{X},P_X;\mathcal{Y}).\]
		\item For any \(f\in L^2\left(\mathcal{X},P_X;\mathcal{Y}\right)\) and any \(n\in\mathbb{N}\),
		\[\iota^*\left(f\right)=\mathbb{E}\left[S^*_\mathbf{X}\left(\left(f\left(X_1\right),...,f\left(X_n\right)\right)^T\right)\right].\]
		\item For any \(f\in\mathcal{H}\) and any \(n\in\mathbb{N}\), \[\iota^*\circ\iota\left(f\right)=\mathbb{E}\left[nS^*_\mathbf{X}\circ S_\mathbf{X}\left(f\right)\right].\]
	\end{enumerate} 
\end{lemma}
\begin{proof}
	\begin{enumerate}[(i)] 
		\item Take any \(f_1\in\mathcal{H}\) and \(f_2\in L^2(\mathcal{X},P_X;\mathcal{Y})\). Then the reproducing property gives
		\[\left\langle\iota f_1,f_2\right\rangle_2=\mathbb{E}\left[\left\langle f_1(X),f_2(X)\right\rangle_\mathcal{Y}\right]=\mathbb{E}\left[\left\langle f_1,K\left(\cdot,X\right)\left(f_2\left(X\right)\right)\right\rangle_\mathcal{H}\right]=\left\langle f_1,\mathbb{E}\left[K\left(\cdot,X\right)\left(f_2\left(X\right)\right)\right]\right\rangle_\mathcal{H}.\]
		\item The fact that \(X_1,...,X_n\stackrel{\text{i.i.d.}}{\sim}X\) and (i) immediately gives
		\[\mathbb{E}\left[S^*_\mathbf{X}\left(\left(f\left(X_1\right),...,f\left(X_n\right)\right)^T\right)\right]=\mathbb{E}\left[\frac{1}{n}\sum^n_{i=1}K\left(\cdot,X_i\right)f(X_i)\right]=\mathbb{E}\left[K\left(\cdot,X\right)f(X)\right]=\iota^*(f).\]
		\item Applying (ii) and the definition of \(S_\mathbf{X}\), 
		\[\iota^*\circ\iota(f)=\mathbb{E}\left[S^*_\mathbf{X}\left(\left(F\left(X_1\right),...,f\left(X_n\right)\right)^T\right)\right]=\mathbb{E}\left[S^*_\mathbf{X}\left(nS_\mathbf{X}\left(f\right)\right)\right]=\mathbb{E}\left[nS^*_\mathbf{X}\circ S_\mathbf{X}\left(f\right)\right].\]
	\end{enumerate}
\end{proof}
Although the inclusion operator \(\iota\) is a compact (in fact, even Hilbert-Schmidt) operator if \(\mathcal{Y}\) is \(\mathbb{R}\) \citep[p. 127, Theorem 4.27]{steinwart2008support}, this is not true in the general case we consider in this report. Indeed, consider the following counterexample, in which \(K(x,x')=k(x,x')\text{Id}\), where \(k(\cdot,\cdot)\) is a bounded scalar kernel with \(k(x_0,x_0)=1\) for some \(x_0\in\mathcal{X}\) and \(\text{Id}:\mathcal{Y}\rightarrow\mathcal{Y}\) is the identity operator. Let \(\{y_i\}_{i=1}^\infty\) be a (countable, by separability assumption) orthonormal basis of \(\mathcal{Y}\). Then \(\{K(x_0,\cdot)y_i\}_{i=1}^\infty\) form a bounded sequence in \(\mathcal{H}\), since
\[\left\lVert K(x_0,\cdot)y_i\right\rVert_\mathcal{H}^2=\left\langle y_i,K(x_0,x_0)y_i\right\rangle_\mathcal{Y}=\left\lVert y_i\right\rVert_\mathcal{Y}^2=1,\]
by the reproducing property. However, the sequence \(\{\iota(K(x_0,\cdot)y_i)\}_{i=1}^\infty\) in \(L^2(\mathcal{X},P_X;\mathcal{Y})\) cannot have a convergent subsequence, since, for any \(i\neq j\), 
\[\left\lVert\iota\left(K\left(x_0,\cdot\right)y_i\right)-\iota\left(K\left(x_0,\cdot\right)y_j\right)\right\rVert^2_2=\mathbb{E}\left[\left\lVert k(x_0,X)y_i-k(x_0,X)y_j\right\rVert_\mathcal{Y}^2\right]=2\mathbb{E}\left[k(x_0,X)^2\right]>0.\]
Hence \(\iota\) is not a compact operator\footnote{See \citet[p.186]{bollobas1999linear} for the definition and equivalent formulations of compact operators. This counterexample does not contradict \citet[Proposition 4.8]{carmeli2006vector}, which says that \(\iota\) is compact if \(K(x,x):\mathcal{Y}\rightarrow\mathcal{Y}\) is compact for all \(x\in\mathcal{X}\) and \(\mathbb{E}\left[\left\lVert K(X,X)\right\rVert_\text{op}\right]<\infty\), since \(K(x,x)=k(x,x)\text{Id}\) is clearly not a compact operator.}. 

The self-adjoint operator \(\iota\circ\iota^*\) is also not compact. Indeed, let \(\{y_i\}_{i=1}^\infty\) be an orthonormal basis of \(\mathcal{Y}\) again, and consider the sequence of functions \(f_i\in L^2(\mathcal{X},P_X;\mathcal{Y})\) given by \(f_i(x)=y_i\) for all \(x\in\mathcal{X}\). Also, consider again the kernel \(K(x,x')=k(x,x')\text{Id}\), where \(k(\cdot,\cdot)\) is a scalar kernel and \(\text{Id}:\mathcal{Y}\rightarrow\mathcal{Y}\) is the identity operator. Then \(\left\lVert f_i\right\rVert^2_2=\mathbb{E}\left[\left\lVert f_i(X)\right\rVert^2_\mathcal{Y}\right]=\left\lVert y_i\right\rVert^2_\mathcal{Y}=1\), so the sequence is bounded, but for any \(i\neq j\),
\begin{alignat*}{2}
	\left\lVert\iota\circ\iota^*(f_i)-\iota\circ\iota^*(f_j)\right\rVert^2_2&=\mathbb{E}_{X_1}\left[\left\lVert\mathbb{E}_{X_2}\left[K(X_1,X_2)f_i(X_2)\right]-\mathbb{E}_{X_2}\left[K(X_1,X_2)f_j(X_2)\right]\right\rVert_\mathcal{Y}^2\right]\\
	&=\mathbb{E}_{X_1}\left[\left\lVert\mathbb{E}_{X_2}\left[k(X_1,X_2)\right]y_i-\mathbb{E}_{X_2}\left[k(X_1,X_2)\right]y_j\right\rVert^2_\mathcal{Y}\right]\\
	&=\mathbb{E}_{X_1}\left[\mathbb{E}_{X_2}\left[k(X_1,X_2)\right]^2\left\lVert y_i-y_j\right\rVert^2_\mathcal{Y}\right]\\
	&=2\mathbb{E}_{X_1}\left[\mathbb{E}_{X_2}\left[k(X_1,X_2)\right]^2\right]\\
	&>0,
\end{alignat*}
using the expression for \(\iota^*\) given in Lemma \ref{Lintegralexpression}(i). So the sequence \(\{\iota\circ\iota^*(f_i)\}^\infty_{i=1}\) in \(L^2(\mathcal{X},P_X;\mathcal{Y})\) cannot have a convergent subsequence, which in turn implies that \(\iota\circ\iota^*\) is not compact. 

\subsection{Regularised Least-Squares Regression}\label{SSrls}
As above, take i.i.d. copies \(\{(X_i,Y_i)\}_{i=1}^n\) of \((X,Y)\). We define the unregularised population, regularised population, unregularised empirical and regularised empirical risk functions with respect to the squared-loss as follows:
\begin{equation}\label{Erisk}
	\begin{split}
		R(f)&=\mathbb{E}\left[\left\lVert f(X)-Y\right\rVert^2_\mathcal{Y}\right];\\
		R_\lambda(f)&=\mathbb{E}\left[\left\lVert f(X)-Y\right\rVert^2_\mathcal{Y}\right]+\lambda\left\lVert f\right\rVert^2_\mathcal{H};\\
		R_n(f)&=\frac{1}{n}\sum^n_{i=1}\left\lVert f(X_i)-Y_i\right\rVert^2_\mathcal{Y};\text{ and}\\
		R_{n,\lambda}(f)&=\frac{1}{n}\sum^n_{i=1}\left\lVert f(X_i)-Y_i\right\rVert^2_\mathcal{Y}+\lambda\left\lVert f\right\rVert^2_\mathcal{H},
	\end{split}
\end{equation}
where \(\lambda>0\) is a regularisation parameter. Here, \(R\) and \(R_n\) is defined for any \(f\in L^2(\mathcal{X},P_X;\mathcal{Y})\), but \(R_\lambda\) and \(R_{n,\lambda}\) are only defined for \(f\in\mathcal{H}\). Also, the population risks \(R\) and \(R_\lambda\) are deterministic functions of \(F\), whereas the empirical risks \(R_n\) and \(R_{n,\lambda}\) are random, varying with the random sample \(\{(X_i,Y_i)\}_{i=1}^n\). 

The following decomposition of the population risk is well-known; see, for example, \citet[Propoisition 1]{cucker2002mathematical}. 
\begin{lemma}\label{Lbiasvariance}
	We have the following decomposition of the risk \(R\):
	\[R(f)=\mathbb{E}\left[\left\lVert f(X)-f^*(X)\right\rVert_\mathcal{Y}^2\right]+R\left(f^*\right).\]
\end{lemma}
\begin{proof}
	See that
	\begin{alignat*}{2}
		R(f)&=\mathbb{E}\left[\left\lVert f(X)-f^*(X)+f^*(X)-Y\right\rVert_\mathcal{Y}^2\right]\\
		&=\mathbb{E}\left[\left\lVert f(X)-f^*(X)\right\rVert_\mathcal{Y}^2\right]+\mathbb{E}\left[\left\lVert f^*(X)-Y\right\rVert^2_\mathcal{Y}\right]+2\mathbb{E}\left[\left\langle f(X)-f^*(X),f^*(X)-Y\right\rangle_\mathcal{Y}\right]\\
		&=\mathbb{E}\left[\left\lVert f(X)-f^*(X)\right\rVert_\mathcal{Y}^2\right]+R(f^*)+2\mathbb{E}\left[\left\langle f(X)-f^*(X),f^*(X)-\mathbb{E}\left[Y\mid X\right]\right\rangle_\mathcal{Y}\right]\\
		&=\mathbb{E}\left[\left\lVert f(X)-f^*(X)\right\rVert_\mathcal{Y}^2\right]+R(f^*),
	\end{alignat*}
	where we applied the law of iterated expectations to go from the second line to the third, and used \(f^*(X)=\mathbb{E}[Y\mid X]\) to go from the third to the last. 
\end{proof}
From Lemma \ref{Lbiasvariance}, it is immediate that \(f^*\) is the minimiser of \(R\) in \(L^2(\mathcal{X},P_X;\mathcal{Y})\). The following lemma formulates the minimisers in \(\mathcal{H}\) of the regularised risks \(R_\lambda\) and \(R_{n,\lambda}\) in terms of the inclusion and evaluation operators. Similar results can be found in many places in the literature, for example \citet[Section 4]{micchelli2005learning} or \citet[p.117, Theorem 5.1]{engl1996regularization}. 
\begin{lemma}\label{Lriskminimisers}
	\begin{enumerate}[(i)]
		\item The minimiser \(f_\lambda\) of the risk \(R_\lambda\) in \(\mathcal{H}\) is unique and is given by
		\[f_\lambda\vcentcolon=\argmin_{f\in\mathcal{H}}R_\lambda(f)=\left(\iota^*\circ\iota+\lambda\textnormal{Id}_\mathcal{H}\right)^{-1}\iota^*f^*=\iota^*\left(\iota\circ\iota^*+\lambda\textnormal{Id}_2\right)^{-1}f^*,\]
		where \(\textnormal{Id}_\mathcal{H}:\mathcal{H}\rightarrow\mathcal{H}\) and \(\textnormal{Id}_2:L^2(\mathcal{X},P_X;\mathcal{Y})\rightarrow L^2(\mathcal{X},P_X;\mathcal{Y})\) are the identity operators. 
		\item The minimiser \(\hat{f}_{n,\lambda}\) of the risk \(R_{n,\lambda}\) in \(\mathcal{H}\) is unique and is given by
		\[\hat{f}_{n,\lambda}\vcentcolon=\argmin_{f\in\mathcal{H}}R_{n,\lambda}(f)=\left(nS^*_\mathbf{X}\circ S_\mathbf{X}+\lambda\textnormal{Id}_\mathcal{H}\right)^{-1}S^*_\mathbf{X}\mathbf{Y}=S^*_\mathbf{X}\left(nS_\mathbf{X}\circ S^*_\mathbf{X}+\lambda\textnormal{Id}_{\mathcal{Y}^n}\right)^{-1}\mathbf{Y},\]
		where \(\textnormal{Id}_{\mathcal{Y}^n}:\mathcal{Y}^n\rightarrow\mathcal{Y}^n\) is the identity operator. 
	\end{enumerate}
\end{lemma}
\begin{proof}
	\begin{enumerate}[(i)]
		\item By Lemma \ref{Lbiasvariance}, we have \(f_\lambda=\argmin_{f\in\mathcal{H}}\tilde{R}_\lambda(f)\), where, for any \(f\in\mathcal{H}\),
		\begin{alignat*}{2}
			\tilde{R}_\lambda(f)&=\mathbb{E}\left[\left\lVert f(X)-f^*(X)\right\rVert^2_\mathcal{Y}\right]+\lambda\left\lVert f\right\rVert^2_\mathcal{H}\\
			&=\left\lVert\iota(f)-f^*\right\rVert^2_2+\lambda\left\lVert f\right\rVert^2_\mathcal{H}.
		\end{alignat*}
		Then \(\tilde{R}_\lambda\) is clearly continuously Fr\'echet differentiable, coercive (Definition \ref{Dcoercive}) and strictly convex (Definition \ref{Dconvex}). So by Lemma \ref{Lglobalextremum}, there exists a unique critical point \(f_\lambda\) that minimises \(\tilde{R}_\lambda\), and by Lemma \ref{Llocalextremum}, at this critical point, we have \(\tilde{R}_\lambda(f_\lambda)=0\). Denote by \(J:L^2(\mathcal{X},L_X;\mathcal{Y})\rightarrow\mathbb{R}\) the map \(f\mapsto\lVert f-f^*\rVert^2_2\); then we have \(J'(f)=2(f-f^*)\) by Lemma \ref{Lfrechetsquare}. Taking the Fr\'echet derivative using Lemma \ref{Lfrechetcomposition}, we have
		\begin{alignat*}{3}
			&&\tilde{R}'_\lambda(f)&=\iota^*\circ J'\circ\iota(f)+2\lambda f\\
			&&&=2\iota^*\left(\iota(f)-f^*\right)+2\lambda f\\
			\implies&\qquad&\iota^*\left(\iota(f_\lambda)-f^*\right)+\lambda f_\lambda&=0\\
			\implies&&\left(\iota^*\circ\iota+\lambda\text{Id}_\mathcal{H}\right)f_\lambda&=\iota^*f^*\\
			\implies&&f_\lambda&=\left(\iota^*\circ\iota+\lambda\text{Id}_\mathcal{H}\right)^{-1}\iota^*f^*,
		\end{alignat*}
		where \((\iota^*\circ\iota+\lambda\text{Id}_\mathcal{H})\) is invertible since \(\iota^*\circ\iota\) is positive and self-adjoint, and \(\lambda>0\). Now see that
		\begin{alignat*}{2}
			\left(\iota^*\circ\iota+\lambda\text{Id}_\mathcal{H}\right)\iota^*\left(\iota\circ\iota^*+\lambda\text{Id}_2\right)^{-1}f^*&=\iota^*\left(\iota\circ\iota^*+\lambda\text{Id}_2\right)\left(\iota\circ\iota^*+\lambda\text{Id}_2\right)^{-1}f^*\\
			&=\iota^*f^*.
		\end{alignat*}
		Apply \((\iota^*\circ\iota+\lambda\text{Id}_\mathcal{H})^{-1}\) to both sides to obtain
		\[f_\lambda=\iota^*\left(\iota\circ\iota^*+\lambda\text{Id}_2\right)^{-1}f^*.\] 
		\item We can write \(R_{n,\lambda}(f)\) as
		\begin{alignat*}{2}
			R_{n,\lambda}(f)&=\frac{1}{n}\sum^n_{i=1}\left\lVert f(X_i)-Y_i\right\rVert_\mathcal{Y}^2+\lambda\lVert f\rVert^2_\mathcal{H}\\
			&=n\left\lVert S_\mathbf{X}(f)-\frac{1}{n}\mathbf{Y}\right\rVert^2_{\mathcal{Y}^n}+\lambda\left\lVert f\right\rVert^2_\mathcal{H}.
		\end{alignat*}
		Then following the same steps as in (i), we take the Fr\'echet derivative of \(R_{n,\lambda}\) and set it to 0 at \(\hat{f}_{n,\lambda}\):
		\begin{alignat*}{3}
			&&R'_{n,\lambda}(f)&=2nS_\mathbf{X}^*\left(S_\mathbf{X}(f)-\frac{1}{n}\mathbf{Y}\right)+2\lambda f\\
			\implies&\qquad&nS^*_\mathbf{X}\left(S_\mathbf{X}(\hat{f}_{n,\lambda})-\frac{1}{n}\mathbf{Y}\right)+\lambda\hat{f}_{n,\lambda}&=0\\
			\implies&&\left(nS^*_\mathbf{X}\circ S_\mathbf{X}+\lambda\text{Id}_\mathcal{H}\right)\hat{f}_{n,\lambda}&=S^*_\mathbf{X}\mathbf{Y}\\
			\implies&&\hat{f}_{n,\lambda}&=\left(nS^*_\mathbf{X}\circ S_\mathbf{X}+\lambda\text{Id}_\mathcal{H}\right)^{-1}S^*_\mathbf{X}\mathbf{Y},
		\end{alignat*}
		where \((nS_\mathbf{X}^*\circ S_\mathbf{X}+\lambda\text{Id}_\mathcal{H})\) is invertible since \(nS_\mathbf{X}^*\circ S_\mathbf{X}\) is positive and self-adjoint, and \(\lambda>0\). 
		
		By the same argument as in (i), we also have
		\[\hat{f}_{n,\lambda}=S^*_\mathbf{X}\left(nS_\mathbf{X}\circ S^*_\mathbf{X}+\lambda\textnormal{Id}_{\mathcal{Y}^n}\right)^{-1}\mathbf{Y}.\]
	\end{enumerate}
\end{proof}

\subsection{Convergence in Probability, Convergence Rates \& Chebyshev's Inequality}\label{SSconvergence}
We are interested in the convergence of \(\hat{f}_{n,\lambda}\) to \(f^*\). The convergence that we will consider in this paper is \textit{convergence in probability}, defined as follows. 
\begin{definition}
	Let \(\{Z_n\}_{n\geq1}\) and \(Z\) be real-valued random variables defined on the same probability space. Then we say that \(\{Z_n\}_{n\geq1}\) \textit{converges in probability} to \(X\), and write \(Z_n\stackrel{P}{\rightarrow}Z\), if, for any \(\epsilon,\delta>0\), 
	\[P\left(\left\lvert Z_n-Z\right\rvert>\epsilon\right)\leq\delta.\]
\end{definition}
We will also be interested in the rate at which \(\hat{f}_{n,\lambda}\) converges to \(f^*\). 
\begin{definition}
	We say that a sequence \(\{Z_n\}_{n\geq1}\) of real-valued random variables is \textit{bounded in probability}, and write \(Z_n=\mathcal{O}_P(1)\), if
	\[\lim_{M\rightarrow\infty}\limsup_{n\rightarrow\infty}P\left(\left\lvert Z_n\right\rvert>M\right)=0.\]
	We write \(Z_n=\mathcal{O}_P(r_n)\) if \(\frac{Z_n}{r_n}=\mathcal{O}_P(1)\). 
\end{definition}
Clearly, if \(r_n\rightarrow0\) as \(n\rightarrow\infty\) and \(Z_n=\mathcal{O}_P(r_n)\), then \(Z_n\stackrel{P}{\rightarrow}0\). 

The simple Chebyshev's inequality is a very well-known and widely-used inequality; for reference, see, for example, \citet[p.8, Corollary 1.2.5]{vershynin2018high}. 
\begin{lemma}[Chebyshev's inequality]
	Let \(Z\) be a real-valued random variable. Then for all \(a>0\), 
	\[P\left(\left\lvert Z\right\rvert\geq a\right)\leq\frac{\mathbb{E}\left[Z^2\right]}{a^2}.\]
\end{lemma}
\begin{proof}
	See that
	\[\mathbb{E}\left[Z^2\right]=\mathbb{E}\left[\mathbf{1}_{\lvert Z\rvert\geq a}Z^2\right]+\mathbb{E}\left[\mathbf{1}_{\lvert Z\rvert<a}Z^2\right]\geq\mathbb{E}\left[\mathbf{1}_{\lvert Z\rvert\geq a}Z^2\right]\geq\mathbb{E}\left[\mathbf{1}_{\lvert Z\rvert\geq a}a^2\right]=a^2P\left(\left\lvert Z\right\rvert\geq a\right),\]
	from which the result follows.
\end{proof}

\section{Universal Consistency}\label{Suniversalconsistency}
Our goal in this section is to investigate the convergence to 0 in probability of
\[R\left(\hat{f}_{n,\lambda}\right)-R\left(f^*\right)=\mathbb{E}\left[\left\lVert\hat{f}_{n,\lambda}(X)-f^*(X)\right\rVert^2_\mathcal{Y}\right]=\left\lVert\iota\hat{f}_{n,\lambda}-f^*\right\rVert^2_2,\]
where the equality comes from Lemma \ref{Lbiasvariance}. We first consider the case where the measure is fixed, i.e. the distributions \(P_{XY}\), \(P_X\) and \(P_Y\), the regression function \(f^*\), the function space \(L^2(\mathcal{X},P_X;\mathcal{Y})\) as well as the operator \(\iota\), are fixed. In Section \ref{SSuniformrates}, we will consider a uniform rate of convergence over a class of distributions. 

We split the above using the triangle inequality into estimation and approximation errors:
\[\left\lVert\iota\hat{f}_{n,\lambda}-f^*\right\rVert_2\leq\left\lVert\iota\hat{f}_{n,\lambda}-\iota f_\lambda\right\rVert_2+\left\lVert\iota f_\lambda-f^*\right\rVert_2.\]
Proposition \ref{PFlambdaconvergestoFstar} shows, under the assumption that \(\mathcal{H}\) is dense in \(L^2(\mathcal{X},P_X;\mathcal{Y})\), the convergence of the second term to 0 as \(\lambda\rightarrow0\), and Proposition \ref{PFhatconvergestoFlambda} shows the convergence of the first term in probability to 0 as \(n\rightarrow\infty\) and \(\lambda\rightarrow0\). Theorem \ref{Tuniversalconsistency} then brings them together to show the consistency of \(\hat{f}_{n,\lambda}\). 
\begin{proposition}[Approximation Error]\label{PFlambdaconvergestoFstar}
	If \(\iota\mathcal{H}\) is dense in \(L^2(\mathcal{X},P_X;\mathcal{Y})\), then \(\lVert f^*-\iota f_\lambda\rVert^2_2\rightarrow0\) as \(\lambda\rightarrow0\). 
\end{proposition}
\begin{proof}
	Take an arbitrary \(\epsilon>0\). By the denseness of \(\iota\mathcal{H}\) in \(L^2(\mathcal{X},P_X;\mathcal{Y})\), there exists some \(f_\epsilon\in\mathcal{H}\) such that \(R(f_\epsilon)-R(f^*)=\lVert\iota f_\epsilon-f^*\rVert^2_2\leq\frac{\epsilon}{2}\). Then see that
	\begin{alignat*}{3}
		\left\lVert f^*-\iota f_\lambda\right\rVert^2_2&=R(f_\lambda)-R(f^*)&&\text{by Lemma \ref{Lbiasvariance}}\\
		&\leq R_\lambda(f_\lambda)-R(f^*)&&\text{since }R_\lambda(f)\geq R(f)\text{ for all }f\in\mathcal{H}\\
		&=R_\lambda(f_\lambda)-R_\lambda(f_\epsilon)+R_\lambda(f_\epsilon)-R(f_\epsilon)+R(f_\epsilon)-R(f^*)\\
		&\leq R_\lambda(f_\epsilon)-R(f_\epsilon)+R(f_\epsilon)-R(f^*)&&\text{since }f_\lambda\text{ minimises }R_\lambda\text{ in }\mathcal{H}\\
		&\leq R_\lambda(f_\epsilon)-R(f_\epsilon)+\frac{\epsilon}{2}&&\text{by the choice of }f_\epsilon\\
		&=\lambda\left\lVert f_\epsilon\right\rVert^2_\mathcal{H}+\frac{\epsilon}{2}&&\text{by the definition of }R_\lambda.
	\end{alignat*}
	Now if \(\lambda\leq\frac{\epsilon}{2\lVert f_\epsilon\rVert^2_\mathcal{H}}\), then
	\[\left\lVert f^*-\iota f_\lambda\right\rVert^2_2\leq\epsilon,\]
	as required.
\end{proof}
\begin{proposition}[Estimation Error]\label{PFhatconvergestoFlambda}
	Take any \(\delta>0\). Then
	\[P\left(\left\lVert\hat{f}_{n,\lambda}-f_\lambda\right\rVert^2_\mathcal{H}\geq\frac{B\mathbb{E}\left[\left\lVert Y\right\rVert^2_\mathcal{Y}\right]}{n\lambda^2\delta}\right)\leq\delta.\]
	In particular, if \(\lambda=\lambda_n\) depends on \(n\) and converges to 0 at a slower rate than \(\mathcal{O}(n^{-1/2})\), then
	\[\left\lVert\hat{f}_{n,\lambda_n}-f_{\lambda_n}\right\rVert^2_\mathcal{H}\stackrel{P}{\rightarrow}0.\]
\end{proposition}
\begin{proof}
	By Lemma \ref{Lriskminimisers}, we can write
	\begin{alignat*}{2}
		\hat{f}_{n,\lambda}-f_\lambda&=(nS^*_\mathbf{X}\circ S_\mathbf{X}+\lambda\text{Id}_\mathcal{H})^{-1}S^*_\mathbf{X}\mathbf{Y}-(nS^*_\mathbf{X}\circ S_\mathbf{X}+\lambda\text{Id}_\mathcal{H})^{-1}(nS^*_\mathbf{X}\circ S_\mathbf{X}+\lambda\text{Id}_\mathcal{H})f_\lambda\\
		&=(nS^*_\mathbf{X}\circ S_\mathbf{X}+\lambda\text{Id}_\mathcal{H})^{-1}\left(S^*_\mathbf{X}\mathbf{Y}-nS^*_\mathbf{X}\circ S_\mathbf{X}f_\lambda-\lambda f_\lambda\right)\\
		&=\left(nS^*_\mathbf{X}\circ S_\mathbf{X}+\lambda\text{Id}_\mathcal{H}\right)^{-1}\left(S^*_\mathbf{X}\mathbf{Y}-nS^*_\mathbf{X}\circ S_\mathbf{X}f_\lambda-\iota^*\left(f^*-\iota f_\lambda\right)\right).\tag{*}
	\end{alignat*}
	Write \(\sigma\) for the spectrum of \(nS^*_\mathbf{X}\circ S_\mathbf{X}\). Then by the spectral theorem for (non-compact) self-adjoint operators \citep[p.141, Theorem 7.12]{hall2013quantum}, there exists a unique projection-valued measure \(\mu\) on the Borel \(\sigma\)-algebra of \(\sigma\) such that
	\[nS^*_\mathbf{X}\circ S_\mathbf{X}=\int_\sigma\gamma d\mu(\gamma),\]
	whence, using the properties of operator-valued integration \citep[p.139, Proposition 7.11]{hall2013quantum} and fact that \(\sigma\subseteq[0,\infty)\) \citep[p.242, Theorem 3.8]{conway1990course}, we can bound its operator norm by
	\[\left\lVert\left(nS^*_\mathbf{X}\circ S_\mathbf{X}+\lambda\text{Id}_\mathcal{H}\right)^{-1}\right\rVert_\textnormal{op}=\left\lVert\int_\sigma\frac{1}{\gamma+\lambda}d\mu(\gamma)\right\rVert_\textnormal{op}\leq\sup_{\gamma\in\sigma}\left\lvert\frac{1}{\gamma+\lambda}\right\rvert\leq\frac{1}{\lambda}.\]
	Then returning to (*) and taking the \(\mathcal{H}\)-norm of both sides, we have
	\[\left\lVert\hat{f}_{n,\lambda}-f_\lambda\right\rVert_\mathcal{H}\leq\frac{1}{\lambda}\left\lVert S^*_\mathbf{X}\mathbf{Y}-nS^*_\mathbf{X}\circ S_\mathbf{X}f_\lambda-\iota^*\left(f^*-\iota f_\lambda\right)\right\rVert_\mathcal{H}.\]
	Hence, for any arbitrary \(\epsilon>0\), by Chebyshev's inequality,
	\begin{alignat*}{2}
		P\left(\left\lVert\hat{f}_{n,\lambda}-f_\lambda\right\rVert_\mathcal{H}\geq\epsilon\right)&\leq P\left(\frac{1}{\lambda}\left\lVert S^*_\mathbf{X}\mathbf{Y}-nS^*_\mathbf{X}\circ S_\mathbf{X}f_\lambda-\iota^*\left(f^*-\iota f_\lambda\right)\right\rVert_\mathcal{H}\geq\epsilon\right)\\
		&\leq\frac{1}{\lambda^2\epsilon^2}\mathbb{E}\left[\left\lVert S^*_\mathbf{X}\mathbf{Y}-nS^*_\mathbf{X}\circ S_\mathbf{X}f_\lambda-\iota^*\left(f^*-\iota f_\lambda\right)\right\rVert_\mathcal{H}^2\right].
	\end{alignat*}
	Here, letting \(Z=S^*_XY-S^*_X\circ S_Xf_\lambda\) and \(Z_i=S^*_{X_i}Y_i-S^*_{X_i}\circ S_{X_i}f_\lambda\), Lemma \ref{Lintegralexpression}(ii) and (iii) tells us that the integral is in fact simply \(\mathbb{E}\left[\left\lVert\frac{1}{n}\sum^n_{i=1}Z_i-\mathbb{E}\left[Z\right]\right\rVert^2_\mathcal{H}\right]\). Hence, 
	\begin{alignat*}{2}
		P\left(\left\lVert\hat{f}_{n,\lambda}-f_\lambda\right\rVert_\mathcal{H}\geq\epsilon\right)&\leq\frac{1}{n\lambda^2\epsilon^2}\mathbb{E}\left[\left\lVert S^*_XY-S^*_X\circ S_Xf_\lambda-\iota^*\left(f^*-\iota f_\lambda\right)\right\rVert_\mathcal{H}^2\right]\\
		&\leq\frac{1}{n\lambda^2\epsilon^2}\mathbb{E}\left[\left\lVert S^*_XY-S^*_X\circ S_Xf_\lambda\right\rVert_\mathcal{H}^2\right]\\
		&\leq\frac{B}{n\lambda^2\epsilon^2}\mathbb{E}\left[\left\lVert Y-f_\lambda\left(X\right)\right\rVert^2_\mathcal{Y}\right],
	\end{alignat*}
	by Lemma \ref{Lbound}(ii). Here, we use the fact that \(f_\lambda\) minimises \(R_\lambda\) in \(\mathcal{H}\), i.e. \(R_\lambda\left(f_\lambda\right)\leq R_\lambda\left(0\right)\), to see that
	\[\mathbb{E}\left[\left\lVert f_\lambda(X)-Y\right\rVert^2_\mathcal{Y}\right]\leq\mathbb{E}\left[\left\lVert f_\lambda(X)-Y\right\rVert^2_\mathcal{Y}\right]+\lambda\left\lVert f_\lambda\right\rVert^2_\mathcal{H}\leq\mathbb{E}\left[\left\lVert Y\right\rVert^2_\mathcal{Y}\right].\]
	Hence,
	\[P\left(\left\lVert\hat{f}_{n,\lambda}-f_\lambda\right\rVert_\mathcal{H}\geq\epsilon\right)\leq\frac{B\mathbb{E}\left[\left\lVert Y\right\rVert^2_\mathcal{Y}\right]}{n\lambda^2\epsilon^2},\]
	from which the result follows. 
\end{proof}
\begin{remark}
	Under additional assumptions on the underlying distribution, we can obtain tighter bounds in Proposition \ref{PFhatconvergestoFlambda}, by using exponential probabilistic inequalities like Bernstein's inequality, instead of Chebyshev's inequality like we did above. This is indeed done, for example, in \citet[Theorem 1]{smale2007learning} for real output spaces and \citet[Theorem 2]{singh2019kernel} for RKHS output spaces in the context of conditional mean embeddings, by assuming that \(Y\) is almost surely bounded, not just square integrable as we assumed in Assumption \ref{Afinitevariance}. 
\end{remark}
\begin{theorem}[Universal Consistency]\label{Tuniversalconsistency}
	Suppose \(\iota\mathcal{H}\) is dense in \(L^2(\mathcal{X},P_X;\mathcal{Y})\). Suppose that \(\lambda=\lambda_n\) depends on the sample size \(n\), and converges to 0 at a slower rate than \(\mathcal{O}(n^{-1/2})\). Then we have
	\[R\left(\hat{f}_{n,\lambda_n}\right)-R\left(f^*\right)=\mathbb{E}\left[\left\lVert\hat{f}_{n,\lambda_n}(X)-f^*(X)\right\rVert^2_\mathcal{Y}\right]=\left\lVert\iota\hat{f}_{n,\lambda_n}-f^*\right\rVert^2_2\stackrel{P}{\rightarrow}0.\]
\end{theorem}
\begin{proof}
	The simple inequality \(\lVert a+b\rVert^2\leq2\lVert a\rVert^2+2\lVert b\rVert^2\) holds in any Hilbert space. Using this, we see that
	\begin{alignat*}{2}
		\left\lVert\iota\hat{f}_{n,\lambda_n}-f^*\right\rVert^2_2&\leq2\left\lVert\iota\hat{f}_{n,\lambda_n}-\iota f_{\lambda_n}\right\rVert^2_2+2\left\lVert\iota f_{\lambda_n}-f^*\right\rVert_2^2\\
		&\leq2B\left\lVert\hat{f}_{n,\lambda_n}-f_{\lambda_n}\right\rVert_\mathcal{H}^2+2\left\lVert\iota f_{\lambda_n}-f^*\right\rVert_2^2,
	\end{alignat*}
	where we used the discussion after Lemma \ref{Lbound} in the last inequality. Here, the second term converges to 0 as \(\lambda_n\rightarrow0\) by Proposition \ref{PFlambdaconvergestoFstar}, and the first term converges in probability to 0 by Proposition \ref{PFhatconvergestoFlambda}. Hence,
	\[\left\lVert\iota\hat{f}_{n,\lambda_n}-f^*\right\rVert_2^2\stackrel{P}{\rightarrow}0\]
	as required.
\end{proof}

\subsection{Uniform Rates in the Well-Specified Case}\label{SSuniformrates}
In our work above, possible bottlenecks are \(\mathbb{E}[\lVert Y\rVert^2_\mathcal{Y}]\) in Proposition \ref{PFhatconvergestoFlambda} being arbitrarily large, or \(f_\epsilon\) in the proof of Proposition \ref{PFlambdaconvergestoFstar} having arbitrarily large norm in \(\mathcal{H}\). In the next result, we consider a class of measures over which the rate of convergence is uniform. In particular, any measure in this class of measures is conditioned to have the conditional expectation \(f^*\) of \(Y\) given \(X\) in \(\mathcal{H}\), i.e. there exists some \(f^*_\mathcal{H}\in\mathcal{H}\) such that \(\iota f^*_\mathcal{H}=f^*\). This is known as the \textit{well-specified} case \citep[p.2]{szabo2016learning}. 
\begin{theorem}\label{Trate}
	For constants \(M,C>0\), define \(\mathcal{P}(M,C)\) as the class of measures such that
	\begin{enumerate}[(i)]
		\item \(\mathbb{E}\left[\left\lVert Y\right\rVert^2_\mathcal{Y}\right]\leq M\), and
		\item \(f^*=\iota f^*_\mathcal{H}\) for some \(f^*_\mathcal{H}\in\mathcal{H}\) with \(\left\lVert f^*_\mathcal{H}\right\rVert^2_\mathcal{H}\leq C\). 
	\end{enumerate}
	Let \(\mathcal{H}\) be dense in \(L^2(\mathcal{X},P_X;\mathcal{Y})\) for all \(P\in\mathcal{P}(M,C)\). Then
	\[\sup_{P\in\mathcal{P}(M,C)}P\left(\left\lVert\iota\hat{f}_{n,\lambda}-f^*\right\rVert^2_2\geq\frac{2B^2M}{n\lambda^2\delta}+2\lambda C\right)\leq\delta.\]
	In particular, if \(\lambda=\lambda_n\) depends on the sample size \(n\) and converges to 0 at the rate of \(\mathcal{O}(n^{-1/4})\), then \(R(\hat{f}_{n,\lambda_n})-R(f^*)=\mathcal{O}_P(n^{-1/4})\) uniformly over the class \(\mathcal{P}(M,C)\) of measures. 
\end{theorem}
\begin{proof}
	First, see that the condition (ii) helps simplify the proof of Proposition \ref{PFlambdaconvergestoFstar}:
	\[\begin{split}
		\left\lVert\iota f_\lambda-f^*\right\rVert^2_2&=R(f_\lambda)-R(f^*_\mathcal{H})\\
		&\leq R_\lambda(f_\lambda)-R_\lambda(f^*_\mathcal{H})+R_\lambda(f^*_\mathcal{H})-R(f^*_\mathcal{H})\\
		&\leq\lambda\left\lVert f^*_\mathcal{H}\right\rVert_\mathcal{H}^2\\
		&\leq\lambda C.
	\end{split}\tag{*}\]
	Then using the inequality \(\lVert\iota\hat{f}_{n,\lambda}-f^*\rVert^2_2\leq2B\lVert\hat{f}_{n,\lambda}-f_\lambda\rVert_\mathcal{H}^2+2\lVert\iota f_\lambda-f^*\rVert^2_2\) as in the proof of Theorem \ref{Tuniversalconsistency}, 
	\begin{alignat*}{3}
		&\sup_{P\in\mathcal{P}(M,C)}P\left(\left\lVert\iota\hat{f}_{n,\lambda}-f^*\right\rVert^2_2>\frac{2B^2M}{n\lambda^2\delta}+2\lambda C\right)\\
		&\leq\sup_{P\in\mathcal{P}(M,C)}P\left(\left\lVert\hat{f}_{n,\lambda}-f_\lambda\right\rVert^2_\mathcal{H}>\frac{BM}{n\lambda^2\delta}\right)+\sup_{P\in\mathcal{P}(M,C)}P\left(\left\lVert f^*-\iota f_\lambda\right\rVert^2_2>\lambda C\right)\\
		&\leq\sup_{P\in\mathcal{P}(M,C)}P\left(\left\lVert\hat{f}_{n,\lambda}-f_\lambda\right\rVert^2_\mathcal{H}\geq\frac{B\mathbb{E}\left[\left\lVert Y\right\rVert_\mathcal{Y}^2\right]}{n\lambda^2\delta}\right)&&\text{by (*)}\\
		&\leq\delta&&\text{by Proposition \ref{PFhatconvergestoFlambda},}
	\end{alignat*}
	as required. 
\end{proof}

\bibliography{ref}
\bibliographystyle{abbrvnat}
\appendix
\section{The Fr\'echet Derivative and Critical Points of Extremum of Nonlinear operators}\label{Sfrechet}
In this appendix, we review some basic theory about the Fr\'echet derivative. We follow the exposition in \citet[Chapter 7]{precup2002methods}. Let \(X\) be a Banach space, \(U\subset X\) an open subset, \(E:U\rightarrow\mathbb{R}\) a functional and \(u\in U\) a given point. Write \(X^*\) for the dual space of \(X\), and write, for any \(x_1\in X\) and \(x_2^*\in X^*\), \(x_2^*(x_1)=(x_2^*,x_1)\). 
\begin{definition}[{\citet[p.97, Definition 7.1]{precup2002methods}}]\label{Dfrechet}
	\(E\) is said to be \textit{Fr\'echet differentiable} at \(u\) if there exists an \(E'(u)\in X^*\) such that
	\[E(u+v)-E(u)=\left(E'(u),v\right)+\omega(u,v)\]
	and
	\[\omega(u,v)=o\left(\left\lvert v\right\rvert\right),\text{ i.e. }\frac{\omega(u,v)}{\left\lvert v\right\rvert}\rightarrow0,\]
	as \(v\rightarrow0\). The element \(E'(u)\) is called the \textit{Fr\'echet derivative} of \(E\) at \(u\). 
\end{definition}
\begin{lemma}[{\citet[p.100, Example 7.2]{precup2002methods}}]\label{Lfrechetsquare}
	Suppose \(X\) is a Hilbert space, and
	\[E(u)=\frac{1}{2}\left\lvert u\right\rvert^2\qquad(u\in X).\]
	Then \(E\) is Fr\'echet differentiable in \(X\), its Fr\'echet derivative \(E':X\rightarrow X^*\) is continuous, and is given by
	\[\left(E'(u),v\right)=(u,v),\qquad v\in X.\]
\end{lemma}
\begin{lemma}[{\citet[p.100, Example 7.3]{precup2002methods}}]\label{Lfrechetcomposition}
	Let \(X\) be a Hilbert space, \(Y\) a Banach space, \(H:X\rightarrow Y\) a bounded linear operator, and \(J:Y\rightarrow\mathbb{R}\) Fr\'echet differentiable in \(Y\). Then the functional \(JH:X\rightarrow\mathbb{R}\) is Fr\'echet differentiable in \(X\), and
	\[(JH)'=H^*J'H,\]
	where \(H^*\) is the adjoint of \(H\). 
\end{lemma}
\begin{lemma}[{\citet[p.103, Proposition 7.2]{precup2002methods}}]\label{Llocalextremum}
	If \(u_0\in U\) is a point of local extremum of \(E\) and \(E\) is Fr\'echet differentiable at \(u_0\), then \(E'(u_0)=0\). 
\end{lemma}
\begin{definition}[{\citet[p.105, Definition 7.4]{precup2002methods}}]\label{Dcoercive}
	A functional \(E:D\subset X\rightarrow\mathbb{R}\) defined on an unbounded set \(E\) is said to be \textit{coercive} if \(E(u)\rightarrow\infty\) as \(\lvert u\rvert\rightarrow\infty\). 
\end{definition}
\begin{definition}[{\citet[p.105, Definition 7.5]{precup2002methods}}]\label{Dconvex}
	Let \(D\) be a convex subset of the Banach space \(X\). A functional \(E:D\rightarrow\mathbb{R}\) is said to be \textit{convex} if
	\[E\left(u+t\left(v-u\right)\right)\leq E(u)+t\left(E(v)-E(u)\right)\]
	for all \(u,v\in D\), \(u\neq v\) and \(t\in(0,1)\). The functional \(E\) is said to be \textit{strictly convex} if strict inequality occurs. 
\end{definition}
Note that, if \(X\) is a Hilbert space and \(E(u)=\frac{1}{2}\lvert u\rvert^2\) as in Lemma \ref{Lfrechetsquare}, \(E\) is clearly coercive and strictly convex. 
\begin{lemma}[{\citet[p.106, Theorem 7.4]{precup2002methods}}]\label{Lglobalextremum}
	Let \(X\) be a reflexive Banach space and \(E:X\rightarrow\mathbb{R}\) be convex, coercive and Fr\'echet differentiable in \(X\). Then there exists \(u_0\in X\) with
	\[E(u_0)=\inf_{u\in X}E(u),\qquad E'(u_0)=0.\]
	If, in addition, \(E\) is strictly convex, then \(E\) has a unique critical point. 
\end{lemma}
We remark that Hilbert spaces are reflexive Banach spaces. 
\end{document}